\documentclass{article} 
\usepackage{iclr2018_conference,times}
\iclrfinalcopy

\usepackage[utf8]{inputenc} 
\usepackage[T1]{fontenc}    
\usepackage{hyperref}       
\usepackage{url}            
\usepackage{booktabs}       
\usepackage{algorithm}
\usepackage{algorithmic}

\usepackage{graphicx}

\usepackage{amsmath}
\usepackage{amsfonts}
\usepackage{amssymb}
\usepackage{amsthm}
\usepackage{bm}
\usepackage{subcaption}

\definecolor{mydarkblue}{rgb}{0,0.08,0.45} 
\hypersetup{
     colorlinks = true,
     citecolor = mydarkblue, 
     linkcolor = mydarkblue, 
     urlcolor = mydarkblue,
     pdftitle={Coulomb GAN},
     pdfauthor={Thomas Unterthiner},
}

\newtheorem{theorem}{Theorem}

\newtheorem*{theorem*}{Theorem}%

\newcommand\Ba{\bm{a}}
\newcommand\Bb{\bm{b}}

\newcommand\Bv{\bm{v}}
\newcommand\Bw{\bm{w}}
\newcommand\Bx{\bm{x}}
\newcommand\By{\bm{y}}
\newcommand\Bz{\bm{z}}

\newcommand\BI{\bm{I}}

\newcommand\Bth{\bm{\theta}}

 \newcommand{\dR}{\mathbb{R}}

 \newcommand{\rd}{\mathrm{d}}

 \newcommand{\cN}{\mathcal{N}}

 \newcommand{\cX}{\mathcal{X}}
\newcommand{\cY}{\mathcal{Y}}

 \newcommand{\Rd}{\mathrm{d}}

\newcommand\sign{\mbox{sign}}
\newcommand\EXP{\mathrm{E}}

\DeclareMathOperator*{\argmin}{arg\,min}

\renewcommand{\leq}{\leqslant}
\renewcommand{\geq}{\geqslant}

\def\EE{\mbox{\bf E}}

\def\di{\bm{\nabla \cdot}}

\def\lap{\bm{\nabla}^2}
\def\nab{\bm{\nabla}}

\title{Coulomb GANs: Provably Optimal Nash Equilibria via Potential Fields}

\author{
 Thomas Unterthiner$^1$
 \And
 Bernhard Nessler$^1$
 \And
 Calvin Seward$^{1,2}$
 \And
 Günter Klambauer$^1$
 \And
 Martin Heusel$^1$
 \And
 Hubert Ramsauer$^1$
 \And
 Sepp Hochreiter$^1$\\
 \And
  \\
 $^1$LIT AI Lab \& Institute of Bioinformatics, Johannes Kepler University Linz, Austria\\
$^2$Zalando Research, Mühlenstraße 25, 10243 Berlin, Germany\\
\texttt{\{unterthiner,nessler,seward,klambauer,mhe,ramsauer,hochreit\}@bioinf.jku.at}\\
}

\begin{document}

\maketitle

\begin{abstract}
Generative adversarial networks (GANs) evolved into one
of the most successful unsupervised techniques for generating
realistic images.
Even though it has recently been shown that GAN training converges, GAN models
often end up in local Nash equilibria that are associated with mode collapse or
otherwise fail to model the target distribution.
We introduce Coulomb GANs, which pose the GAN learning problem
as a potential field,
where generated samples are attracted to training set samples but
repel each other.
The discriminator learns a potential field while the generator
decreases the energy by moving its samples
along the vector (force) field determined by the gradient of the
potential field.
Through decreasing the energy, the GAN model learns
to generate samples according to the whole target
distribution and does not only cover some of its modes.
We prove that Coulomb GANs possess only one
Nash equilibrium which is optimal in the sense that
the model distribution equals the target distribution.
We show the efficacy of Coulomb GANs on LSUN bedrooms, CelebA faces, CIFAR-10
and the Google Billion Word text generation.
\end{abstract}

\section{Introduction}
\label{sec:introduction}

Generative adversarial networks (GANs) \citep{Goodfellow:14nips}
excel at constructing realistic images
\citep{Radford:15,Ledig:16,Isola:16,Arjovsky:17,Berthelot:17}
and text \citep{Gulrajani:17}.
In GAN learning, a discriminator network guides the learning of another, generative network.
This procedure can be considered as a game between the generator
which constructs synthetic data
and the discriminator which separates
synthetic data from training set data \citep{Goodfellow:17tutorial}.
The generator's goal is to construct data which the discriminator cannot
tell apart from training set data.
GAN convergence points are local Nash equilibria. At these local Nash equilibria
neither the discriminator nor the generator can locally improve its
objective.

Despite their recent successes, GANs have several problems.
First (I), until recently it
was not clear if in general gradient-based GAN learning could
converge to one of the local Nash
equilibria \citep{Salimans:16,Goodfellow:14criteria,Goodfellow:14nips}.
It is even possible to construct counterexamples \citep{Goodfellow:17tutorial}.
Second (II), GANs suffer from ``mode collapsing'',
where the model generates samples only in certain regions which are called modes.
While these modes contain realistic samples,
the variety is low and only a few prototypes are generated.
Mode collapsing is less likely
if the generator is trained with batch normalization,
since the network is bound to create a certain variance among its
generated samples within one batch \citep{Radford:15,Chintala:16}.
However batch normalization
introduces fluctuations of normalizing
constants which can be harmful \citep{Klambauer:17,Goodfellow:17tutorial}.
To avoid mode collapsing without batch normalization,
several methods have been proposed
\citep{Che:17,Metz:16,Salimans:16}.
Third (III), GANs cannot assure that the density of training samples is
correctly modeled by the generator.
The discriminator only tells the generator whether a
region is more likely to contain samples from the training set or synthetic samples.
Therefore the discriminator can only distinguish the support
of the model distribution
from the support of the target distribution.
Beyond matching the support of distributions, GANs with proper objectives
may learn to locally align model and target densities
via averaging over many training examples.
On a global scale, however,
GANs fail to equalize model and target densities.
The discriminator does not inform the generator globally where
probability mass is missing. 
Consequently, standard GANs are not assured to capture the global sample
density and are prone
to neglect large parts of the target
distribution. The next paragraph gives an example of this.
Fourth (IV), the discriminator of GANs may forget previous modeling
errors of the generator which then may reappear, a property that leads
to oscillatory behavior instead of convergence \citep{Goodfellow:17tutorial}.

Recently, problem (I) was solved by proving that GAN learning does indeed
converge when discriminator and generator are learned using a two time-scale
learning rule \citep{Heusel:17}.
Convergence means that the expected SGD-gradient of both the
discriminator objective and the generator objective are zero. Thus, neither the
generator nor the discriminator can locally improve, i.e., learning has reached a local
Nash equilibrium.
However, convergence alone does not guarantee
good generative performance. It is possible to converge to sub-optimal
solutions which are local Nash equilibria.
Mode collapse is a special case of a
local Nash equilibrium associated with sub-optimal generative performance.
For example, assume a two mode real world
distribution where one mode contains too few and the other mode too many generator samples. If
no real world samples are between these two distinct modes, then the discriminator penalizes to
move generated samples outside the modes. Therefore the generated samples cannot be correctly
distributed over the modes. Thus, standard GANs cannot capture the global sample density such that
the resulting generators are prone to neglect large parts of the real world distribution. A more detailed
example is listed in the Appendix in Section~\ref{sec:modecollapse-example}.

In this paper, we introduce a novel GAN model, the Coulomb GAN, which
has only one Nash equilibrium. We are later going to show that this
Nash equilibrium is optimal, i.e., the model
distribution matches the target distribution.
We propose Coulomb GANs to avoid the
GAN shortcoming (II) to (IV) by using a potential field created by point charges
analogously to the electric field in physics.
The next section will introduce the idea of learning in a potential field and prove
that its only solution is optimal. We will then show how learning the discriminator
and generator works in a Coulomb GAN and discuss the assumptions needed
for our optimality proof. In Section~\ref{sec:experiments} we will then see that
the Coulomb GAN does indeed work well in practice and that the samples it produces
have very large variability and appear to capture the original distribution very well.

\paragraph{Related Work.}
Several GAN approaches have been suggested for bringing the target and model distributions in
alignment using not just local discriminator information:
Geometric GANs combine samples via a linear
support vector machine which uses the discriminator outputs as samples,
therefore they are much more robust to mode collapsing \citep{Lim:17}.
Energy-Based GANs \citep{Zhao:16} and their later improvement BEGANs \citep{Berthelot:17}
optimize an energy landscape based on auto-encoders.
McGANs match mean and covariance of synthetic and target data, therefore
are more suited than standard GANs to approximate
the target distribution \citep{Mroueh:17}.
In a similar fashion, Generative Moment Matching Networks \citep{Li:15}
and MMD nets \citep{Dziugaite:15} directly optimize a generator network to match a
training distribution by using a loss function based on the maximum mean discrepancy (MMD) criterion \citep{Gretton:12}.
These approaches were later expanded to include an MMD criterion with learnable kernels and discriminators \citep{Li:17}.
The MMD criterion that these later approaches optimize has a form similar
to the energy function that Coulomb GANs optimize (cf. Eq.~\eqref{eq:objectiveF}).
However, all MMD approaches end up using either Gaussian or Laplace
kernels, which are not guaranteed to find the optimal solution where the
model distribution matches the target distribution.
In contrast, the
Plummer kernel which is employed in this work has been shown to lead
to the optimal solution \citep{Hochreiter:05kernel}.
We show that even a simplified version of the Plummer kernel, the
low-dimensional Plummer kernel, ensures that gradient descent
convergences to the optimal solution as stated by Theorem~\ref{th:convergence}.
Furthermore, most MMD GAN approaches use the MMD directly as loss
function though the number of possible samples in a mini-batch is limited.
Therefore MMD approaches face a sampling
problem in high-dimensional spaces.
The Coulomb GAN instead learns a discriminator network that gradually
improves its approximation of the potential field
via learning on many mini-batches.
The discriminator network also tracks the slowly
changing generator distribution during learning.
Most importantly however, our approach is,
to the best of our knowledge, the first one for which
optimality, i.e., ability to perfectly learn a target distribution, can be proved.

The use of the Coulomb potential for learning is not new.
Coulomb Potential Learning was proposed to store arbitrary many patterns in a potential
field with perfect recall and without spurious patterns \citep{Perrone:95}.
Another related work is the
Potential Support Vector Machine (PSVM), which minimizes Coulomb potential differences
\citep{Hochreiter:01cltech,Hochreiter:03c}.
\citet{Hochreiter:05kernel} also used a potential function
based on Plummer kernels for optimal unsupervised learning, on which we base our work on Coulomb GANs.

\begin{figure}
\centering
\includegraphics[width=\textwidth]{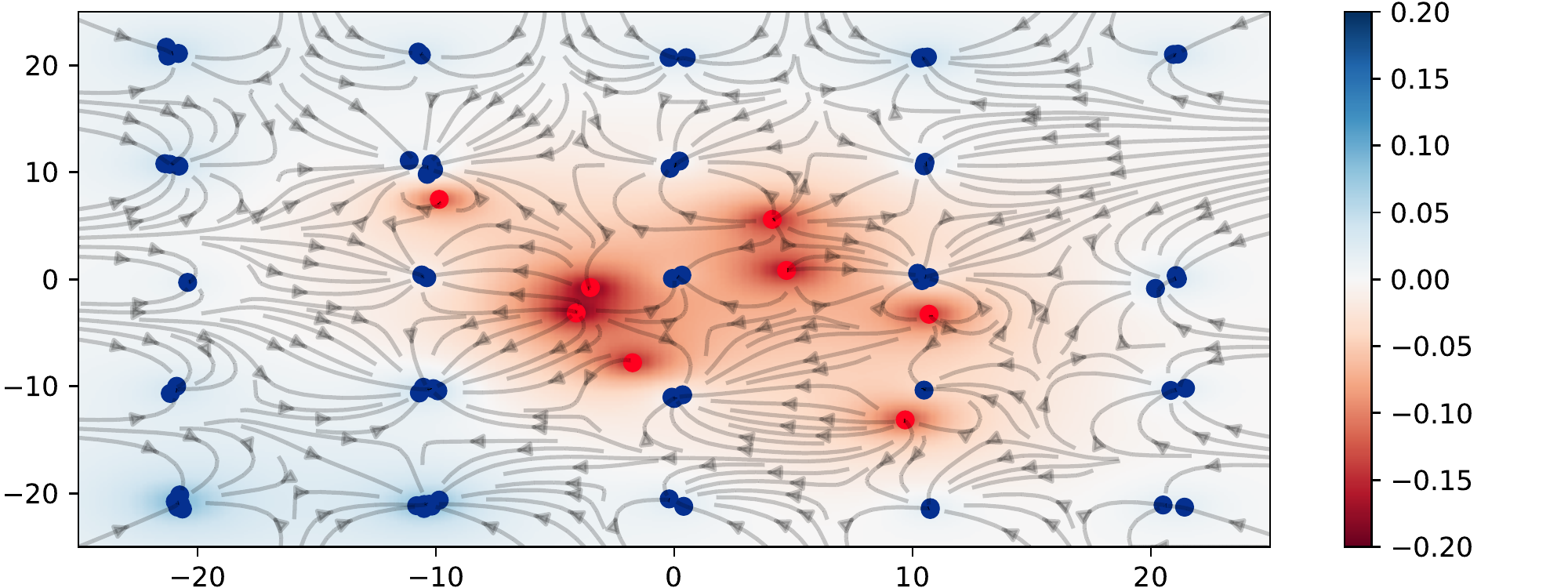}
\caption{The vector field of a Coulomb GAN.
The basic idea behind the Coulomb GAN: true samples (blue) and
generated samples (red) create a potential field (scalar field).
Blue samples act as sinks that attract the red samples, which
repel each other. The superimposed vector field
shows the forces acting on the generator samples
to equalize potential differences, and the background color
shows the potential at each position. Best viewed in color.}
\label{fig:field}
\end{figure}

\section{Coulomb GANs}
\label{sec:CoulombGANs}
\subsection{General Considerations on GANs}
We assume data samples $\Ba \in \dR^m$ for
a model density $p_x(.)$ and a target density $p_y(.)$.
The goal of GAN learning is to modify the
model in a way to obtain $p_x(.)=p_y(.)$.
We define the {\bf difference of densities}
$\rho(\Ba) = p_y(\Ba) -  p_x(\Ba)$ which should be pushed toward zero
for all $\Ba \in \dR^m$ during learning.
In the GAN setting, the discriminator $D(\Ba)$ is a function $D:\dR^m \to \dR$
that learns to discriminate between generated
and target samples and predicts how
likely it is that $\Ba$ is sampled from the target distribution.
In conventional GANs, $D(\Ba)$ is usually optimized to approximate the probability of seeing a
target sample, or $\rho(\Ba)$ or some similar function.
The generator $G(\Bz)$ is a continuous function $G: \dR^n \to \dR^m$ which maps some
$n$-dimensional random variable $\Bz$ into the space of target
samples.
$\Bz$ is typically sampled from a multivariate Gaussian or Uniform distribution.

In order to improve the generator, a GAN uses the gradient of the discriminator
$\nab_{a} D(\Ba)$ with respect to the discriminator input $\Ba=G(\Bz)$ for learning.
The objective of the generator is a scalar function $D(G(\Bz))$,
therefore the gradient of the objective function is just a scaled
version of the gradient $\nab_{a} D(\Ba)$ which would then propagate further
to the parameters of $G$.
This gradient $\nab_{a} D(\Ba)$ tells the generator
in which direction $\rho(\Ba)$ becomes larger, i.e., in which direction
the ratio of target examples increases.
The generator changes slightly
so that $\Bz$ is now mapped to a
new $\Ba'=G'(\Bz)$, moving the sample generated by $\Bz$ a little bit towards the direction where $\rho(\Ba)$ was larger,
i.e., where target examples were more likely.
However, $\rho(\Ba)$ and its derivative only take
into account the local neighborhood of
$\Ba$, since regions of the sample space that are
distant from $\Ba$ do not have much influence on $\rho(\Ba)$.
Regions of data space that have strong support in $p_y$ but not
in $p_x$ will not be noticed by the generator via discriminator gradients.
The restriction to local environments hampers GAN learning significantly \citep{Arjovsky:17Towards, Arjovsky:17}.

The theoretical analysis of GAN learning can be done at three different levels:
(1) in the space of distributions $p_x$ and $p_y$ regardless of the fact that $p_x$ is realized by $G$ and $p_z$,
(2) in the space of functions $G$ and $D$ regardless of the fact that $G$ and $D$ are typically realized
by a parametric form, i.e., as neural networks,
or (3) in the space of the parameters of $G$ and $D$.
\citet{Goodfellow:14nips} use (1) to prove convergence of GAN learning in their Proposition 2
in a hypothetical scenario where the learning algorithm operates by making small, local moves in $p_x$
space.
In order to see that level (1) and (2) should both be understood as hypothetical scenarios, remember that
in all practical implementations, $p_x$ can only be altered implicitly by making small changes to the
generator function G, which in turn can only be changed
implicitly by small steps in its parameters.
Even if we assume that the mapping from a distribution $p_x$ to the generator $G$ that induced it
exists and is unique, this mapping from $p_x$ to the space of $G$ is not continuous.
To see this, consider changing a distribution $p1_x$ to a new distribution $p2_x$ by
moving a small amount $\epsilon$ of its density to an isolated region in space where $p1_x$ has no support.
Let's further assume this region has distance $d$ to any other regions of support of $p1_x$.
By letting $\epsilon \to 0$, the distance between $p1_x$ and $p2_x$ becomes smaller, yet the distance
between the inducing generator functions $G_1$ and $G_2$
(e.g. using the supremum norm on bounded functions)
will not tend to zero
because for at least one function input $\Bz$ we have: $|G_1(\Bz) - G_2(\Bz)| \geq d$.
Because of this, we need to go further than the distribution space when analyzing GAN learning.
In practice, when learning GANs, we are restricted to small steps in parameter space, which in turn lead to small
steps in function space and finally to small steps in distribution space. But not all small steps in distribution
space can be realized this way as shown in the example above.
This causes local Nash equilibria in the function space,
because even though in distribution space it would be easy to escape by making small steps, such a step would require
very large changes in function space and is thus not realizable.
In this paper we show that Coulomb GANs
do not exhibit any local Nash equilibria in the space of the functions $G$ and $D$.
To the best of our knowledge, this is the first formulation of GAN learning that can guarantee this property.
Of course, Coulomb GANs are learned as parametrized neural networks,
and as we will discuss in Subsection~\ref{sec:coulombgans-in-practice},
Coulomb GANs are not immune to the usual issues that
arise from parameter learning, such as over- and underfitting, which can cause
local Nash Equilibria due to a bad choice of parameters.

\subsection{From Conventional GANs to Potentials}
If the density $p_x(.)$ or $p_y(.)$ approaches a Dirac
delta-distribution, gradients vanish since the density approaches
zero except for the exact location of data points.
Similarly, electric point charges are often represented by Dirac
delta-distributions, however the electric potential created by a point charge
has influence
everywhere in the space, not just locally.
The electric potential (Coulomb potential) created by the point charge $Q$ is
$\Phi_C = \frac{1}{4\pi \varepsilon_{0}} \frac{Q}{r}$,
where $r$ is the distance to the location of $Q$ and
$\varepsilon_{0}$ is the dielectric constant.
Motivated by this electric potential, we introduce a similar
concept for GAN learning: Instead of the difference of densities
$\rho(\Ba)$, we rather consider a {\bf potential function} $\Phi(\Ba)$
defined as
\begin{align}
\label{eq:potential}
\Phi(\Ba) \ &= \
\int \rho(\Bb) \ k(\Ba,\Bb) \ \Rd\Bb \ ,
\end{align}
with some kernel $k\left(\Ba,\Bb \right)$ which defines the
influence of a point at $\Bb$ onto a point at $\Ba$.
The crucial advantage of potentials $\Phi(\Ba)$ is that each point
can influence each other point in space if $k$ is chosen properly.
If we minimize this potential $\Phi(\Ba)$ we are at the same time minimizing
the difference of densities $\rho(\Ba)$:
For all kernels $k$ it holds that
if $\rho(\Bb)=0$ for all $\Bb$ then $\Phi(\Ba)=0$ for all $\Ba$.
We must still show that (i) $\Phi(\Ba)=0$ for all $\Ba$ then $\rho(\Bb)=0$ for all $\Bb$,
and even more importantly, (ii) whether a gradient optimization of $\Phi(\Ba)$
leads to  $\Phi(\Ba)=0$ for all $\Ba$. This is not the case for every kernel.
Indeed only for particular kernels $k$
gradient optimization of $\Phi(\Ba)$ leads to $\rho(\Bb)=0$ for all
$\Bb$, that is, $p_x(\Bb)=p_y(\Bb)$ for all $\Bb$
\citep{Hochreiter:05kernel} (see also Theorem~\ref{th:convergence}
below). An example for such a kernel $k$ is the one leading to the Coulomb potential $\Phi_C$ from above, where
$k\left(\Ba,\Bb \right)=\frac{1}{\|\Ba -\Bb \|}$ for $m=3$.
As we will see in the following, the ability to have samples that influence each
other over long distances, like charges in a Coulomb potential, will lead to
GANs with a single, optimal Nash equilibrium.

\subsection{GANs as Electrical Fields}
For Coulomb GANs,
the generator objective is derived from electrical field dynamics: real
and generated samples generate a potential field,
where samples of the same class (real vs.\ generated)
repel each other, but attract samples of the opposite class.
However, real data points are fixed in space, so the only samples that can
move are the generated ones.
In turn, the gradient of the potential with respect to the input
samples creates a vector field in the space of samples.
The generator can move its samples along the forces generated by this field.
Such a field is depicted in Figure~\ref{fig:field}.
The discriminator learns to predict the potential function,
in order to approximate the current potential landscape of
all samples, not just the ones in the current mini-batch.
Meanwhile, the generator learns to
distribute its samples
across the whole field
in such a way that the energy is minimized,
thus naturally avoids
mode collapse and covering the whole region of support of the data.
The energy is minimal and equal to zero only
if all potential differences are zero and the model distribution
is equal to the target distribution.

Within an electrostatic field,
the strength of the force on one particle depends on its distance
to other particles and their charges.
If left to move freely, the particles will organize
themselves into a constellation where all forces equal out and no
potential differences are present.
For continuous charge distributions, the potential field is constant
without potential differences
if charges no longer move since forces are equaled out.
If the potential field is constant, then the difference of densities $\rho$
is constant, too. Otherwise the potential field would have local bumps.
The same behavior is modeled within
our Coulomb GAN, except that real and generated
samples replace the positive and negative particles, respectively,
and that the real data points remain fixed.
Only the generated samples are allowed to move freely, in order to minimize $\rho$.
The generated samples are attracted by real samples, so they move towards them.
At the same time, generated samples should repel each other, so they
do not clump together, which would lead to mode collapsing.

Analogously to electrostatics, the potential $\Phi(\Ba)$
from Eq.~\eqref{eq:potential} gives rise to
a {\bf field} $\EE(\Ba)=  -\nab_{a} \Phi(\Ba)$.
and to an {\bf energy function} $F\left( \rho \right) = \frac{1}{2} \int \rho(\Ba) \Phi(\Ba) \Rd\Ba$.
The field $\EE(\Ba)$ applies a force on charges at $\Ba$ which pushes
the charges toward lower energy constellations.
Ultimately, the Coulomb GAN aims to make the potential $\Phi$ zero
everywhere via the field $\EE(\Ba)$, which is the negative gradient of $\Phi$.
For proper kernels $k$, it can be shown that (i) $\Phi$ can be pushed
to zero via its negative gradient given by the field and (ii) that
$\Phi(\Ba)=0$ for all $\Ba$ implies $\rho(\Ba)=0$ for all $\Ba$,
therefore, $p_x(\Ba)=p_y(\Ba)$ for all $\Ba$
\citep{Hochreiter:05kernel} (see also Theorem~\ref{th:convergence} below).

\subsubsection{Learning Process}
During learning we do not change $\Phi$ or $\rho$ directly.
Instead, the location $\Ba=G(\Bz)$ to which the random variable $\Bz$ is mapped
changes to a new location $\Ba'=G'(\Bz)$.
For the GAN optimization dynamics, we assume that generator samples $\Ba=G(\Bz)$
can move freely, which is ensured by a sufficiently complex generator.
Importantly, generator samples originating from random variables $\Bz$
do neither disappear nor are they newly created but are conserved.
This conservation is expressed by
the continuity equation \citep{Schwartz:87} that describes how the difference
between distributions $\rho(\Ba)$ changes as the particles are moving along the field,
i.e., how moving samples during the learning process changes our densities:
\begin{align}
\dot{\rho}(\Ba) \ &= \  - \ \di (\rho(\Ba) \ \Bv(\Ba))
\end{align}
for sample density difference $\rho$ and unit charges that move with ``velocity''
$\Bv(\Ba) = \sign(\rho(\Ba)) \EE(\Ba)$.
The continuity equation is crucial as it
establishes the connection between moving
samples and changing the generator density and thereby $\rho$.
The sign function of the velocity
indicates whether positive or negative charges are present at $\Ba$.
The divergence operator ``$\di$'' determines whether samples move toward
or outward of $\Ba$ for a given field.
Basically, the continuity equation says that if the generator density increases,
then generator samples must flow into the region and if the generator density decreases,
they flow outwards.
We assume that differently charged particles cancel each other.
If generator samples are moved
away from a location $\Ba$ then $\rho(\Ba)$ is increasing while $\rho(\Ba)$ is
decreasing when generator samples are moved toward $\Ba$.
The continuity equation is also obtained as a first order ODE to move particles
in a potential field \citep{Dembo:88}, therefore describes the dynamics
how the densities are changing.
We obtain
\begin{align}
\label{eq:learnrho}
\dot{\rho}(\Ba) \ &= \
- \ \sign(\rho(\Ba))
\ \di(\rho(\Ba) \ \EE(\Ba))
\ = \
- \ \di(|\rho(\Ba)| \ \EE(\Ba)) \ .
\end{align}
The density difference $\rho(\Ba)$ indicates
how many samples are locally available for being moved.
At each local minimum and local maximum $\Ba$ of $\rho$ we obtain
$\nab_{a} \rho(\Ba) = 0$.
Using the product rule for the divergence
operator, at points $\Ba$ that are minima or maxima, Eq.~\eqref{eq:learnrho} reduces to
\begin{align}
\dot{\rho}(\Ba) \ &= \
- \ \sign(\rho(\Ba))
\ \rho(\Ba) \ \di \EE(\Ba) \ .
\end{align}
In order to ensure that $\rho$ converges to zero,
it is necessary and sufficient that
$\sign (\di \EE(\Ba)) =  \sign( \rho(\Ba) )$, where $\div \Ba \rho(\Ba) = 0$,
as this condition ensures the uniform decrease of the maximal absolute density differences $|\rho(\Ba_{\max})|$.

\subsubsection{Choice of Kernel}
As discussed before, the choice of kernel is crucial for Coulomb GANs.
The $m$-dimensional Coulomb kernel and the $m$-dimensional
Plummer kernel lead to (i) $\Phi$ that is pushed
to zero via the field it creates and (ii) that
$\Phi(\Ba)=0$ for all $\Ba$ implies $\rho(\Ba)=0$ for all $\Ba$,
therefore, $p_x(\Ba)=p_y(\Ba)$ for all $\Ba$
\citep{Hochreiter:05kernel}. Thus, gradient learning with these kernels
has been proved to converge to an optimal solution.
However, both the $m$-dimensional Coulomb and the $m$-dimensional
Plummer kernel lead to numerical instabilities if $m$ is large.
Therefore the Coulomb potential $\Phi(\Ba)$ for the Coulomb GAN
was constructed by
a low-dimensional Plummer kernel $k$ with parameters $d\leq m-2$ and
$\epsilon$:
\begin{align}
\label{eq:Plummer}
\Phi(\Ba) \ &= \
\int \rho(\Bb) \ k\left(\Ba,\Bb \right) \ \Rd\Bb\ , \quad
k(\Ba,\Bb) \ = \
\frac{1}{(\sqrt{\|\Ba -\Bb\|^2 \ + \ \epsilon^2})^d} \ .
\end{align}
The original Plummer kernel is obtained with $d=m-2$.
The resulting field and potential energy is
\begin{align}
\label{eq:field}
\EE(\Ba)   \ &= \ - \ \int \rho(\Bb) \
\nab_{a} k\left(\Ba,\Bb \right) \ \Rd\Bb \ = \ - \ \nab_{a} \
               \Phi\left(\Ba\right) \ , \\
\label{eq:energy}
F\left( \rho \right) \ &= \ \frac{1}{2}
\int \rho(\Ba) \ \Phi\left(\Ba\right) \ \Rd\Ba \ = \ \frac{1}{2}
\int \int \ \rho(\Ba) \ \rho(\Bb) \
k\left(\Ba,\Bb \right) \ \Rd\Bb \ \Rd\Ba \ .
\end{align}
The next theorem states
that for freely moving generated samples, $\rho$ converges
to zero, that is, $p_x(.)=p_y(.)$, when using this potential function $\Phi(\Ba)$.
\begin{theorem}[Convergence with low-dimensional Plummer kernel]
\label{th:convergence}
For $\Ba,\Bb \in \dR^m$, $d\leq m-2$, and $\epsilon>0$
the densities $p_x(.)$ and $p_y(.)$ equalize over time
when minimizing energy $F$ with the low-dimensional Plummer kernel
by gradient descent.
The convergence is faster for larger $d$.
\end{theorem}
\begin{proof} See Section~\ref{sec:proof_th1}. \end{proof}

\subsection{Definition of the Coulomb GAN}
The Coulomb GAN minimizes the electric potential energy from Eq.~\eqref{eq:field} using
a stochastic gradient descent based approach using mini-batches.
Appendix Section~\ref{sec:samplebased-equations} contains the equations for the Coulomb potential, field,
and energy in this case.
Generator samples are obtained by drawing $N_x$ random numbers $\Bz_i$ and transforming them
into outputs $\Bx_i = G(\Bz_i)$. Each mini-batch also includes $N_y$ real world samples $\By_i$.
This gives rise to a mini-batch specific potential, where in Eq.~\eqref{eq:Plummer}
we use $\rho(\Ba) = p_y(\Ba) - p_x(\Ba)$ and replace the expectations
by empirical means using the drawn samples:
\begin{align}
\label{eq:objectivePhi}
\hat{\Phi}(\Ba) \ &= \ \frac{1}{N_y} \sum_{i=1}^{N_y} k\left(\Ba,\By_i \right) \ - \
\frac{1}{N_x} \sum_{i=1}^{N_x} k\left(\Ba,\Bx_i \right) .
\end{align}
It is tempting to have a generator network that directly minimizes this potential $\hat{\Phi}$
between generated and training set points.
In fact, we show that $\hat{\Phi}$ is an unbiased estimate for $\Phi$ in Appendix Section~\ref{sec:samplebased-equations}.
However, the estimate has very high variance: for example, if a mini-batch
fails to sample training data from an existing mode, the field would drive all generated
samples that have been generated at this mode to move elsewhere. The high variance
has to be counteracted by extremely low learning rates, which makes learning infeasible in
practice, as confirmed by initial experiments.
Our solution
to this problem is to have a network that generalizes over the mini-batch specific
potentials: each mini-batch contains different
generator samples $\mathcal{X} = \Bx_i$ for $i=1,\ldots,N_x$ and
real world samples $\mathcal{Y} = \By_i$ for $i=1,\ldots,N_y$, they
create a batch-specific potential $\hat{\Phi}$. 
The goal of the discriminator is to learn
$\EXP_{\mathcal{X,Y}}(\hat{\Phi}(\Ba)) = \Phi(\Ba)$, i.e.,
the potential averaged over many mini-batches.
Thus the discriminator function $D$ fulfills a similar role as other typical GAN discriminator
functions, i.e., it discriminates between real and generated data such
that for any point in space $\Ba$, $D(\Ba)$ should be greater than
zero if the $p_y(\Ba) >  p_x(\Ba)$ and smaller than zero otherwise.
In particular $D(\Ba)$ also indicates, via
its gradient and its potential properties,
directions toward regions where training set samples are predominant and
where generator samples are predominant.

The generator in turn tries to move all of its samples according to
the vector field into areas where generator samples are missing and
training set samples are predominant.
The generator minimizes the approximated energy $F$ as predicted by the
discriminator. The loss  $\mathcal{L}_D$ for the discriminator and $\mathcal{L}_G$
for the generator are given by:
\begin{align}
\label{eq:disObj}
\mathcal{L}_D(D ; G) \ &= \
   \frac{1}{2} \ \EXP_{p_a}\left((D(\Ba) \ - \ \hat{\Phi}(\Ba))^2\right) \\
\mathcal{L}_G(G ; D) \ &= \ - \ \frac{1}{2} \ \EXP_{p_z}\left( D(G(\Bz)) \right) \ .
\end{align}
Where $p(\Ba) = 1/2 \int \cN(\Ba ; G(\Bz) ,\epsilon \BI) p_z(\Bz) \rd \Bz
+ 1/2 \int \cN(\Ba ; \By ,\epsilon \BI) p_y(\By) \rd \By$,
i.e., a distribution where each point of support both of the generator
and the real world distribution
is surrounded with a Gaussian ball of width $\epsilon \BI$ similar to \cite{Bishop:98gtm},
in order to overcome the problem that the generator distribution is only a sub-manifold of $\dR^m$.
These loss functions cause the approximated potential values
$D(\Ba)$ that are negative are pushed toward zero.
Finally, the Coulomb GAN, like all other GANs, consists of two parts: a generator to generate
model samples, and a discriminator that provides its learning signal.
Without a discriminator, our would be very similar to GMMNs \citep{Li:15},
as can be seen in Eq.~\eqref{eq:objectiveF}, but with an optimal Kernel specifically tailored
to the problem of estimating differences between probability distributions.

We use each mini-batch only for one update of the discriminator
and the generator. It is important to note that the discriminator
uses each sample in the mini batch twice: once as a point to generate
the mini-batch specific potential $\hat{\Phi}$, and once as a point
in space for the
evaluation of the potential $\hat{\Phi}$ and its approximation $D$.
Using each sample twice is done for performance reasons, but not
strictly necessary:
the discriminator could learn
the potential field by sampling points that lie between generator and real samples as in \cite{Gulrajani:17}, but we are mainly interested
in correct predictions in the vicinity of generator samples. Pseudocode for
the learning algorithm is detailed in Algorithm~\ref{alg:coulombgan} in the appendix.

\subsubsection{Optimality of the Solution}
Convergence of the GAN learning process was
proved for a two time-scales update rule by \citet{Heusel:17}.
A local Nash equilibrium is
a pair of
generator and discriminator $(D^*, G^*)$
that fulfills the two conditions
\begin{align*}
D^* = \argmin_{D \in U(D^*)}\mathcal{L}_D(D ; G^*) \quad \text{and} \quad
G^* = \argmin_{G \in U(G^*)}\mathcal{L}_G(G ; D^*)\ .
\end{align*}
for some neighborhoods $U(D^*)$ and $U(G^*)$.
We show in the following Theorem~\ref{th:optimality} that for Coulomb GANs
every local Nash equilibrium necessarily is identical to the unique global Nash equilibrium.
In other words, any equilibrium point of the Coulomb GAN
that is found to be local optimal has to be the one global Nash equilibrium
as the minimization
of the energy $F(\rho)$ in Eq.~\eqref{eq:objectiveF} leads to a single, global optimum at $p_y = p_x$.

\begin{theorem}[Optimal Solution]
\label{th:optimality}
If the pair $(D^*, G^*)$ is a local Nash equilibrium for the
Coulomb GAN objectives, then it is
the global Nash equilibrium, and no other local Nash equilibria exist,
and
$G^*$ has output distribution $p_x = p_y$.
\end{theorem}
\begin{proof}
See Appendix Section~\ref{sec:proof_th2}.
\end{proof}

\subsubsection{Coulomb GANs in Practice}\label{sec:coulombgans-in-practice}
To implement GANs in practice, we need learnable models for $G$ and $D$.
We assume that our models for $G$ and $D$ are continuously differentiable
with respect to their parameters and inputs.
Toward this end, GANs are typically implemented as neural networks optimized by
(some variant of) gradient descent.
Thus we may not find the optimal $G^*$ or $D^*$,
since neural networks  may suffer from capacity or optimization issues.
Recent research indicates that the effect of local minima in deep learning
vanishes with increasing depth \citep{Dauphin:14, Choromanska:15, Kawaguchi:16}, such
that this limitation becomes less restrictive as our ability to train deep networks grows thanks
to hardware and optimization improvements.

The main problem with learning Coulomb GANs is to approximate the potential
function $\Phi$, which is a complex function in a high-dimensional space,
since the potential can be very non-linear and non-smooth.
When learning the discriminator,
we must ensure that enough data is sampled and averaged
over.
We already lessened the non-linear function problem by
using a low-dimensional Plummer kernel. But still, this kernel can
introduce large non-linearities if samples are close to each other.
It is crucial that the discriminator learns slow enough to accurately estimate the
potential function which is induced by the current generator.
The generator, in turn, must be even slower since it must be tracked by the
discriminator.
These approximation problems are supposed to be tackled by the
research community in near future, which would enable optimal GAN learning.

The formulation of GAN learning as a potential field naturally solves the
mode collapsing issue: the example described in Section~\ref{sec:modecollapse-example},
where a normal GAN cannot get out of a local Nash equilibria is not a converged
solution for the Coulomb GAN: If all probability mass of the generator lies in
one of the modes, then both attracting forces from real-world samples located
at the other mode as well as repelling forces from the over-represented
generator mode will act upon the generator until it generates samples at
the other mode as well.

\section{Experiments}
\label{sec:experiments}
In all of our experiments, we used a low-dimensional
Plummer Kernel of dimensionality $d=3$. This kernel
both gave best computational performance and has low risk of running into numerical
issues. We used a batch size of 128. To evaluate the quality of a GAN, the FID metric
as proposed by \citet{Heusel:17} was calculated by using 50k
samples drawn from the generator, while the training set statistics were calculated
using the whole training set. We compare to BEGAN\,\citep{Berthelot:17}, DCGAN\,\citep{Radford:15} and WGAN-GP\,\citep{Gulrajani:17} both in their original
version as well as when using the two-timescale update-rule (TTUR), using the settings from \citet{Heusel:17}.
We additionally compare to MMD-GAN\,\citep{Li:17}, which is conceptually very similar to the Coulomb GAN, but
uses a Gaussian Kernel instead of the Plummer Kernel. We use the dataset-specific settings recommended in \,\citep{Li:17}
and report the highest FID score over the course of training.
All images shown in this paper were produced with a random seed and not cherry picked.
The implementation used for these experiments is available online\footnote{
\ificlrfinal
\url{www.github.com/bioinf-jku/coulomb_gan}
\else
\url{filled.in.after.review} 
\fi
}.
The appendix Section~\ref{sec:gaussianmixture} contains an
additional toy example demonstrating
that Coulomb GANs do not suffer from mode collapse when fitting a simple Gaussian Mixture
of 25 components.

\subsection{Image Datasets}
To demonstrate the ability of the Coulomb GAN to learn distributions in high
dimensional spaces, we trained a Coulomb GAN on several popular image data sets:
The cropped and centered images
of celebrities from the Large-scale CelebFaces Attributes (\emph{``CelebA''})
 data set \citep{Liu:15}, the \emph{LSUN bedrooms} data set consists of over 3 million 64x64 pixel
images of the bedrooms category of the large scale image database LSUN
\citep{Yu:15} as well as the CIFAR-10 data set.
For these experiments, we used the DCGAN architecture \citep{Radford:15}
with a few modifications: our convolutional kernels all have a kernel size
of 5x5, our random seed that serves
as input to the generator has fewer dimensions: 32 for CelebA and LSUN bedrooms,
and 16 for CIFAR-10. Furthermore, the discriminator uses twice as many feature
channels in each layer as in the DCGAN architecture.
For the Plummer kernel, $\epsilon$ was set to 1.
We used the Adam optimizer with a learning
rate of $10^{-4}$ for the generator and $5\cdot10^{-5}$ for the discriminator.
To improve convergence performance, we used the $\tanh$ output
activation function \citep{LeCun:98}.
For regularization we used an L2 weight decay term with a weighting
factor of $10^{-7}$.
Learning was stopped by monitoring the FID metric \citep{Heusel:17}. Once
learning plateaus, we scaled the learning rate down by a factor of 10 and let
it continue once more until the FID plateaus. The results are reported in
Table~\ref{tbl:imageresults}, and generated images can be seen in Figure~\ref{fig:image-examples}
and in the Appendix in Section~\ref{sec:more-examples}.
Coulomb GANs tend to outperform standard GAN approaches like BEGAN and DCGAN, but
are outperformed by the Improved Wasserstein GAN. However it is important to note
that the Improved Wasserstein GAN used a more advanced network architecture based
on ResNet blocks\,\citep{Gulrajani:17}, which we could not replicate due to runtime constraints.
Overall, the low FID of Coulomb GANs stem from the fact that the images show a wide variety
of different samples. E.g. on CelebA, Coulomb GAN exhibit a very wide variety of
faces, backgrounds, eye colors and orientations.
\begin{figure}[h]
\centering
    \includegraphics[width=0.8\textwidth]{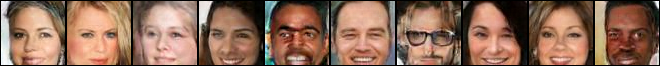} \\
    \includegraphics[width=0.8\textwidth]{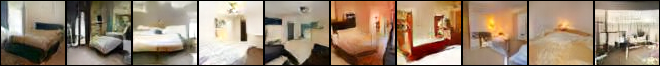} \\
    \includegraphics[width=0.8\textwidth]{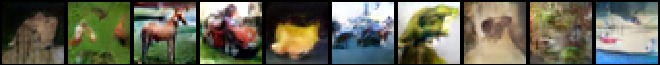}
    \caption{Images from a Coulomb GAN after training on CelebA (first row),
    LSUN bedrooms (second row) and CIFAR 10 (last row). Further examples are located in
    the appendix in Section~\ref{sec:more-examples}}
\label{fig:image-examples}
\end{figure}
\begin{figure}[h]
\centering
    \includegraphics[width=0.8\textwidth]{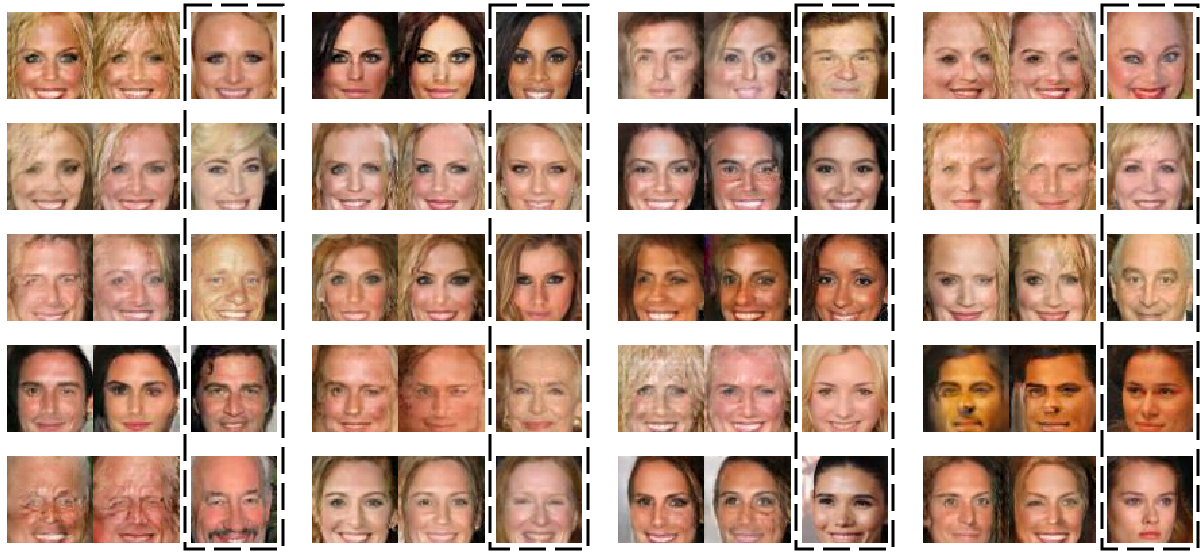}
    \caption{The most similar pairs found in batches of 1024 generated
      faces sampled from the Coulomb GAN,
      and the nearest neighbor from the training data shown as third
      image.
      Distances were calculated as Euclidean distances on pixel level.}
\label{fig:celeba-duplicates}
\end{figure}
To further investigate how much variation the samples generated by the
Coulomb GAN contains, we followed the advice of
Arora and Zhang \citep{Arora:17birthdaytest} to estimate
the support size of the generator's distribution by checking how large a sample
from the generator must be before we start generating duplicates.
We were able to generate duplicates with a probability of around 50\,\%
when using samples of size 1024, which indicates that the support size
learned by the Coulomb GAN would be around 1M.
This is a strong indication that the Coulomb GAN was able to
spread out its samples over the whole target distribution.
A depiction is included in Figure~\ref{fig:celeba-duplicates},
which also shows the nearest neighbor in the training set
of the generated images, confirming that the Coulomb GAN does not just memorize training images.

\subsection{Language Modeling}
We repeated the experiments from \citet{Gulrajani:17}, where Improved Wasserstein GANs (WGAN-GP) were trained
to produce text samples after being trained on the Google Billion Word data set \citep{Chelba:13},
using the same network architecture as in the original publication. We use the Jensen-Shannon-divergence
on 4-grams and 6-grams as an evaluation criterion. The results are summarized in Table~\ref{tbl:wordresults}.
\begin{table}[htp]
\small
\begin{subfigure}[t]{0.31\textwidth}
\begin{tabular}{crrrrrrr}
 \toprule
 data set  & WGAN-GP & ours\\
 \midrule
 4 grams  & 0.38 / 0.35     &   0.35  \\
 6 grams  & 0.77 / 0.74     &   0.74  \\
 ~\\
 \bottomrule
 \end{tabular}
 \caption[Results DCGAN]{Normalized
Jensen-Shanon-Divergence for the Google Billion Word data. Values for WGAN-GP are
without/with TTUR, taken from \citet{Heusel:17}.}
 \label{tbl:wordresults}
\end{subfigure}~\qquad
\begin{subfigure}[t]{0.6\textwidth}
    \centering
        \begin{tabular}{lcccrr}
         \toprule
         data set  & BEGAN & DCGAN & WGAN-GP & MMD & ours\\
         \midrule
         CelebA   & 29.2 / 28.5  &  21.4 / 12.5 &  ~4.8 / ~4.2 &   63.2 &  9.3  \\
         LSUN     & 113 / 112 &  70.4 / 57.5 &  20.5 / ~9.5 &   94.9 & 31.2  \\
         CIFAR10 &       -       &       -      &  29.3 / 24.8 &   38.2 & 27.3  \\
         \bottomrule
         \end{tabular}
        \caption[Results DCGAN]{Performance comparison in FID (lower is better)
         on different data sets. Values for all methods except Coulomb GAN and MMD-GAN
        are from without/with TTUR, taken from \citet{Heusel:17}.}
        \small
         \label{tbl:imageresults}
\end{subfigure}
\end{table}

\section{Conclusion}
The Coulomb GAN is a generative adversarial network with strong theoretical guarantees.
Our theoretical results show that the Coulomb GAN will be able to approximate
the real distribution perfectly if the networks have sufficient capacity and training does
not get stuck in local minima.
Our results show that the potential field used by the Coulomb GAN far outperforms
MMD based approaches due to its low-dimensional Plummer kernel, which is better suited for
modeling probability density functions, and is very
effective at eliminating the mode collapse problem in GANs.
This is because our loss function forces the generated
samples to occupy different regions of the learned distribution.
In practice, we have found that Coulomb GANs are able to produce a wide range
of different samples. However, in our experience,
this sometimes leads to a small number of generated samples that are
non-sensical interpolations of existing data modes.
While these are sometimes also present in other GAN models
\citep{Radford:15},
we found that our model produces such
images at a slightly higher rate.
This issue might be solved by finding better ways of learning the discriminator,
as learning the correct potential field is crucial for the Coulomb GAN's performance.
We also observed that increasing the
capacity of the discriminator seems to always increase the generative performance.
We thus hypothesize that the largest issue in learning Coulomb GANs
is that the discriminator needs to approximate the potential field $\Phi$ very
well in a high-dimensional space.
In summary, instead of directly optimizing
a criterion based on local differences of densities which can exhibit many
local minima, Coulomb GANs are based on a potential field that has no
local minima. The potential field is created by point charges
in an analogy to electric field in physics.
We have proved that if learning converges then it converges
to the optimal solution if the samples can be moved freely.
We showed that Coulomb GANs avoid mode collapsing, model the target distribution
more truthfully than standard GANs,
and do not overlook high probability regions of the
target distribution.

\ificlrfinal
\section*{Acknowledgments}
This work was supported by Zalando SE with Research Agreement 01/2016,
Audi.JKU Deep Learning Center, Audi Electronic Venture GmbH, IWT research grant IWT150865
(Exaptation), H2020 project grant 671555 (ExCAPE) and NVIDIA Corporation. The authors
would like to thank Philipp Renz for fruitful discussions.
\fi

\bibliography{Coulomb_GAN}
\bibliographystyle{iclr2018_conference}

\newpage
\appendix
\section{Appendix}
\subsection{Example of Convergence to Mode Collapse in Conventional GANs}
\label{sec:modecollapse-example}
As an example of how a GAN can converge to a Nash Equilibrium that exhibits
mode collapse, consider a target distribution
that consists of two distinct/non-overlapping regions of support $C_1$ and $C_2$ that are distant from each other,
i.e., the target probability is zero outside of $C_1$ and $C_2$.
Further assume that 50\,\% of the probability mass is in $C_1$ and
50\,\% in $C_2$.
Assume that the the generator has mode-collapsed onto $C_1$,
which contains 100\,\% of the generator's probability mass.
In this situation, the optimal discriminator classifies all points
from $C_2$ as ''real`` (pertaining to the target distribution) by supplying an output of $1$ for them ($1$ is the
target for real samples and $0$ the target for generated samples).
Within $C_1$, the other region,
the discriminator sees twice as many generated
data points as real ones, as 100\,\% of the
probability mass of the generator's distribution is in $C_1$, but only 50\,\% of the
probability mass of the real data distribution.
So one third of the points seen by the discriminator in $C_1$ are real,
the other 2 thirds are generated.
Thus, to minimize its prediction error for a proper objective
(squared or cross entropy), the discriminator
has to output $1/3$ for every point from $C_1$.
The optimal output is even independent
of the exact form of the real distribution in $C_1$. The generator
will match the shape of the target distribution locally. If the
shape is not matched,
local gradients of the discriminator with respect to its input would
be present and the generator would improve locally. If local
improvements of the generator are no longer possible, the shape of the
target distribution is matched and the discriminator output is
locally constant.
In this situation,
the expected gradient of the discriminator is the zero vector, because it has
reached an optimum.
Since the discriminator output is constant in $C_1$ (and $C_2$), the generator's
expected gradient is the zero vector, too.
The situation is also stable
even though we still have random fluctuations from the ongoing
stochastic gradient (SGD) learning:
whenever the generator produces data outside of (but close to) $C_1$,
the discriminator can easily detect this and push
the generator's samples back.
Inside $C_1$, small deviations of the generator from the shape of the real distribution are detected
by the discriminator as well, by deviating slightly from
$1/3$.
Subsequently, the generator is pushed back to the original shape.
If the discriminator deviates from its optimum, it will also be forced back to its optimum.
So overall, the GAN learning reached a local Nash equilibrium and has
converged in the sense that the parameters fluctuate around the
attractor point (fluctuations depend on learning rate, sample size, etc.).
To achieve true mathematical convergence, \citet{Heusel:17}
assume decaying learning rates to anneal the random fluctuations,
similar to \citet{Robbins:51} original
convergence proof for SGD.

\subsection{Proof of Theorem 1}
\label{sec:proof_th1}

We first recall Theorem~\ref{th:convergence}:
\begin{theorem*}[Convergence with low-dimensional Plummer kernel]
For $\Ba,\Bb \in \dR^m$, $d\leq m-2$, and $\epsilon>0$
the densities $p_x(.)$ and $p_y(.)$ equalize over time
when minimizing energy $F$ with the low-dimensional Plummer kernel
by gradient descent.
The convergence is faster for larger $d$.
\end{theorem*}

 In a first step, we prove that for local maxima or local
 minima $\Ba$ of $\rho$, the expression
 $\sign (\di \EE(\Ba)) =  \sign( \rho(\Ba) )$ holds for $\epsilon$
 small enough.
 For proving this equation, we apply the Laplace operator for spherical
 coordinates to the low-dimensional Plummer kernel.
 Using the result, we see that the integral
 $\di \EE(\Ba) =  -  \int \rho(\Bb) \lap_{a} k\left(\Ba,\Bb \right) \Rd\Bb$
 is dominated by large negative values of $\lap_{a} k$
 around $\Ba$. These negative values can even
 be decreased by decreasing $\epsilon$.
 Therefore we can ensure by a small enough $\epsilon$ that
 at each local minimum and local maximum $\Ba$ of $\rho$
 $\sign(\dot{\rho}(\Ba)) =  -  \sign(\rho(\Ba))$.
 Thus, the maximal and minimal points of $\rho$ move toward zero.

 In a second step, we show that new maxima or minima cannot appear and
 that the movement of $\Phi$ toward zero stops at zero and not earlier.
 Since $\rho$ is continuously differentiable, all
 points in environments of maxima and minima move toward zero.
 Therefore the largest $|\rho(\Ba)|$ moves toward zero.
 We have to ensure that moving toward zero does not converge to a
 point apart from zero.
 We derive that the movement toward zero is lower bounded by
 $\dot{\rho}(\Ba) = - \sign(\rho(\Ba))  \lambda  \rho^2(\Ba)$.
 Thus, the movement slows down at $\rho(\Ba)=0$. Solving the
 differential equation and applying it to the maximum of the absolute
 value of $\rho$ gives
 $|\rho|_{\max} (t) = 1/(\lambda  t  +   (|\rho|_{\max}(0))^{-1} )$.
 Thus, $\rho$ converges to zero over time.

\begin{proof}
For $d=m-2$, we have $\lap k(\Ba,\Bb)=\delta(\Ba-\Bb)$, where the
theorem  has already been proved
for $\epsilon$ small enough \citep{Hochreiter:05kernel}.

At each local minimum and local maximum $\Ba$ of $\rho$ we have
$\nab_{a} \rho(\Ba) = 0$.
Using the product rule for the divergence
operator, Eq.~\eqref{eq:learnrho} reduces to
\begin{align}
\dot{\rho}(\Ba) \ &= \
- \ \sign(\rho(\Ba))
\ \rho(\Ba) \ \di \EE(\Ba) \ .
\end{align}
The term $\di \EE(\Ba)$ can be expressed as
\begin{align}
\di \EE(\Ba)   \ &= \ - \ \lap_{a} \Phi \left( \Ba \right)
\ = \ - \ \int \rho(\Bb) \ \lap_{a} k\left(\Ba,\Bb \right) \ \Rd\Bb \ .
\end{align}

We next consider $\lap_{a} k\left(\Ba,\Bb \right)$ for the
low-dimensional Plummer kernel.
We define the {\em spherical Laplace operator} in
$(m-1)$ dimensions as $\lap_{S^{m-1}}$, then the Laplace operator in
spherical coordinates is (Proposition~2.5 in Frye \& Efthimiou
\citep{Efthimiou:14}):
\begin{align}
\label{eq:polarLap}
\lap \ &= \ \frac{\partial^2}{\partial r^2} \ + \
\frac{m-1}{r} \ \frac{\partial}{\partial r} \ + \
\frac{m-1}{r^2} \ \lap_{S^{m-1}} \ .
\end{align}
Note that  $\lap_{S^{m-1}}$ only has second order derivatives with
respect to the angles of the spherical coordinates.

With $r=\|\Ba -\Bb\|$ we obtain for the Laplace operator applied to
the low-dimensional Plummer kernel:
\begin{align}
\label{eq:LapPlummer}
\lap k(\Ba,\Bb) \ &= \
d \ (- \ \epsilon^2 \ m \ + \ (2 \ + \ d \ - \ m) \ r^2) \ (\epsilon^2
\ + \ r^2)^{-2 - d/2}  \ .
\end{align}
and in particular
\begin{align}
\label{eq:LapPlummer0}
\lap k(\Ba,\Ba) \ &= \
- \ m \ d \epsilon^{-(d+2)} \ .
\end{align}
For $d\leq m-2$ we have $(2  +  d  -  m)\leq0$, and obtain
\begin{align}
\lap k(\Ba,\Bb) \ < \ 0 \ ,
\end{align}
and
\begin{align}
\frac{\partial }{\partial r} \lap k(\Ba,\Bb) \ &= \
d \ (2 + d) \ r\ (\epsilon^2 \ (2 + m) \ + \ (-2 - d + m) r^2)
 \ (\epsilon^2 + r^2)^{-3 - d/2}  \ > \ 0
\end{align}
and
\begin{align}
\frac{\partial }{\partial \epsilon} \lap k(\Ba,\Bb) \ &= \
d \ (2 + d) \ \epsilon
\ (\epsilon^2 m \ + \ (-4 - d + m) r^2)\
(\epsilon^2 + r^2)^{-3 - d/2} \ > \ 0\ .
\end{align}
Therefore, $\lap k(\Ba,\Bb)$ is negative with minimum
$- m d \epsilon^{-(d+2)}$ at $r=0$ and
increasing with $r$ and increasing with $\epsilon$ for $d\leq m-4$.
For $d=m-3$ we have to restrict in the following the sphere
$S_{\tau}(\Ba)$ to $\tau < \sqrt{m} \epsilon$ and
ensure increase of $\lap k(\Ba,\Bb)$ with $\epsilon$.

If $\rho(\Bb)\not=0$, then
we define a sphere $S_{\tau}(\Ba)$ with radius $\tau$ around $\Ba$ for
which holds $\sign(\rho(\Bb))=\sign(\rho(\Ba))$
for each $\Bb \in S_{\tau}(\Ba)$. Note that $\lap k(\Ba,\Bb)$ is
continuous differentiable.
We have
\begin{align}
\label{eq:lapK}
\di \EE(\Ba)  \ &= \
- \ \int \rho(\Bb) \ \lap_{a} k\left(\Ba,\Bb \right) \ \Rd\Bb
\ = \\ \nonumber
&- \ \int_{S_{\tau}(\Ba)} \rho(\Bb) \
\lap_{a} k\left(\Ba,\Bb \right) \ \Rd\Bb
\ - \ \int_{T\setminus S_{\tau}(\Ba)} \rho(\Bb) \
\lap_{a} k\left(\Ba,\Bb \right) \ \Rd\Bb \ .
\end{align}
We bound $\lap k(\Ba,\Bb)$ by
\begin{align}
0 \ > \ \lap k(\Ba,\Bb) \ &= \
d \ (- \ \epsilon^2 \ m \ + \ (2 \ + \ d \ - \ m) \ r^2) \ (\epsilon^2
\ + \ r^2)^{-2 - d/2}  \ > \
d \  (2 \ + \ d \ - \ m) \ r^{-2 - d}  \ .
\end{align}
Using $\tau$, we now bound
$\left|\int_{T\setminus S_{\tau}(\Ba)} \rho(\Bb) \ \lap_{a}
k\left(\Ba,\Bb \right) \ \Rd\Bb \right|$ independently from
$\epsilon$, since $\rho$ is a difference of distributions.
For small enough $\epsilon$ we can ensure
\begin{align}
\left|\int_{S_{\tau}(\Ba)} \rho(\Bb) \ \lap_{a} k\left(\Ba,\Bb \right)
  \ \Rd\Bb\right| \ &> \
\left|\int_{T\setminus S_{\tau}(\Ba)} \rho(\Bb) \ \lap_{a} k\left(\Ba,\Bb \right) \ \Rd\Bb \right|\ .
\end{align}
Therefore we have
\begin{align}
\sign (\di \EE(\Ba))  \ &= \ \sign( \rho(\Ba) ) \ .
\end{align}

Therefore we have
at each local minimum and local maximum $\Ba$ of $\rho$
\begin{align}
\sign(\dot{\rho}(\Ba)) \ &= \ - \ \sign(\rho(\Ba)) \ .
\end{align}
Therefore the maximal and minimal points of $\rho$ move toward zero.
Since $\rho$ is continuously differentiable as is the field,
also the points in an environment of the maximal and minimal points
move toward zero.
Points that are not in an environment of the maximal or minimal points
cannot become maximal points in an infinitesimal time step.

Since the contribution of $\Ba$ environment
$S_{\tau}(\Ba)$ dominates the integral
Eq.~\eqref{eq:lapK},
for $\epsilon$ small enough there exists a
positive $0<\lambda$ globally for all minima and maxima as well
as for all time steps for which holds:
\begin{align}
| \di \EE(\Ba) | \ &> \ \lambda \ | \rho(\Ba)| \ .
\end{align}
The factor $\lambda$ depends on $k$
and on the initial $\rho$.
$\lambda$ is proportional to $d$.
Larger $d$ lead to larger $| \di \EE(\Ba) |$ since the maximum or
minimum $\rho(\Ba)$ is upweighted.
There might exist initial conditions $\rho$ for which $\lambda \to 0$,
e.g.\ for infinite many maxima and minima, but they are impossible
in our applications.

Therefore maximal or minimal points approach zero faster or equal
than given by
\begin{align}
\dot{\rho}(\Ba) \ &= \
- \ \sign(\rho(\Ba)) \ \lambda \ \rho^2(\Ba)  \ .
\end{align}
In particular this differential equation dominates
the global maximum $|\rho|_{\max}$ of $|\rho(.)|$.
Solving the differential equation gives that at least
\begin{align}
|\rho|_{\max} (t) \ &= \
\frac{1}{\lambda \ t \ + \  \left(|\rho|_{\max}(0)\right)^{-1} } \ .
\end{align}
Thus $d$ influences the worst case rate of convergence, where larger $d$
with $d\leq m-2$ leads to faster worst case convergence.

Consequently, $\rho$ converges to the zero function over time, that
is,  $p_x(.)$ becomes equal to $p_y(.)$.
\end{proof}

\subsection{Proof of Theorem 2}
\label{sec:proof_th2}
We first recall Theorem~\ref{th:optimality}:
\begin{theorem*}[Optimal Solution]
If the pair $(D^*, G^*)$ is a local Nash equilibrium for the
Coulomb GAN objectives, then it is
the global Nash equilibrium, and no other local Nash equilibria exist,
and
$G^*$ has output distribution $p_x = p_y$.
\end{theorem*}

\begin{proof}
$(D^*, G^*)$ being in a local Nash equilibrium means that
$(D^*, G^*)$ fulfills the two conditions
\begin{align}
D^* = \argmin_{D \in U(D^*)}\mathcal{L}_D(D ; G^*) \quad \text{and} \quad
G^* = \argmin_{G \in U(G^*)}\mathcal{L}_G(G ; D^*)
\end{align}
for some neighborhoods $U(D^*)$ and $U(G^*)$.
For Coulomb GANs that means, $D^*$ has learned the potential $\Phi$ induced by $G^*$
perfectly, because $\mathcal{L}_D$ is
convex in $D$, thus if $D^*$ is optimal within an neighborhood $U(D^*)$, it must be the global optimum.
This means that $G^*$ is directly minimizing
$\mathcal{L}_G(G ; D) = -\frac{1}{2} \EXP_{p_z}\left(\Phi(G(\Bz))
\right)$.
The Coulomb potential energy is according to Eq.~\eqref{eq:energy}
\begin{align}
F\left( \rho \right) \ &= \ \frac{1}{2} \int \rho(\Ba) \Phi(\Ba)
\Rd\Ba \ = \ \frac{1}{2} \int p_y(\Ba) \Phi(\Ba)
\Rd\Ba \ - \ \frac{1}{2} \int p_x(\Ba) \Phi(\Ba)
\Rd\Ba \ .
\end{align}
Only the samples from $p_x$ stem from the generator, where
$p_x(\Ba)=\int \delta(\Ba-G(\Bz))p_z(\Bz)\Rd \Bz$.
Here $\delta$ is the $\delta$-distribution centered at zero.
The part of the energy which depends on the generator is
\begin{align}
&- \ \frac{1}{2} \ \int p_x(\Ba) \ \Phi(\Ba) \
\Rd\Ba \ = \ - \
\frac{1}{2} \ \int \ \int \delta(\Ba-G(\Bz)) \ p_z(\Bz) \ \Rd \Bz \ \Phi(\Ba)
\ \Rd\Ba \\ \nonumber
&= \ - \ \frac{1}{2} \ \int \ \big(\int \delta(\Ba-G(\Bz)) \ \Phi(\Ba)\
\Rd\Ba \big)  \ p_z(\Bz) \ \Rd \Bz   \\ \nonumber
&= \ - \ \frac{1}{2} \ \int p_z(\Bz) \ \Phi(G(\Bz)) \ \Rd \Bz \ = \
-\frac{1}{2} \ \EXP_{p_z}\left(\Phi(G(\Bz)) \right)\ .
\end{align}
Theorem~\ref{th:convergence}
guarantees that there are no other local minima except the global one
when minimizing $F$.
$F$ has one minimum, $F = 0$, which implies $\Phi(\Ba)=0$ and
$\rho(\Ba)=0$ for all $\Ba$, therefore also $p_y = p_x$
according to Theorem~\ref{th:convergence}. Each $\Phi(\Ba)\not=0$
would mean there exist potential differences which in turn
would cause forces on generator samples that allow to further
minimize the energy.
Since we assumed that the generator can reach the minimum $p_y = p_x$ for
any $p_y$, it will be reached by local (stepwise) optimization of
$-\frac{1}{2} \EXP_{p_z}\left(\Phi(G(\Bz)) \right)$ with respect to $G$.
Since the pair $(D^*, G^*)$ is optimal within their neighborhood, the
generator has reached this minimum as there is not other local minimum
than the global one.
Therefore $G^*$ has model density $p_x$ with $p_y = p_x$.
The convergence point is a global Nash equilibrium,
because there is no approximation error and zero energy $F = 0$
is a global minimum for discriminator and generator, respectively.
Theorem~\ref{th:convergence} ensures that other local Nash equilibria
are not possible.
\end{proof}

\subsection{Coulomb Equations in the case of Finite Samples}
\label{sec:samplebased-equations}
GANs are sample-based, that is, samples
are drawn from the model for learning \citep{Hochreiter:05kernel,Gutmann:12}.
Typically this is done in mini-batches, where each mini-batch
consists of two sets of samples, the target samples $\mathcal{Y} = \{\By_i | i = 1 \ldots N_y\}$,
and the model samples
$\mathcal{X} = \{\Bx_i | i = 1 \ldots N_x\}$.

For such finite samples, i.e. point charges, we have to use delta
distributions to obtain unbiased estimates of the the model distribution $p_x(.)$ and the
target distribution $p_y(.)$:
\begin{small}
\begin{align}
\hat{p}_y(\Ba; \cY) &= \frac{1}{N_y} \sum_{i=1}^{N_y}
\delta\left(\Ba - \By_i \right) \ , \quad
\hat{p}_x(\Ba; \cX) = \frac{1}{N_x} \sum_{i=1}^{N_x}
\delta\left(\Ba - \Bx_i \right) \ , \quad
\hat{\rho}(\Ba; \cX, \cY) \ &= \ p_y(\Ba; \cY)\ - \ p_x(\Ba; \cX) \ ,
\end{align}
\end{small}
where $\delta$ is the Dirac $\delta$-distribution centered at zero.
These are unbiased estimates of the underlying distribution, as can be seen by:
\begin{small}
\begin{align}
\EXP_{\cX} \left(\frac{1}{N_x} \sum_{i=1}^{N_x} \delta(\Ba-\Bx_i) \right) \ = \
\frac{1}{N_x} \sum_{i=1}^{N_x} \EXP_{x_i}\left(\delta(\Ba-\Bx_i) \right) \ = \
\frac{1}{N_x} \sum_{i=1}^{N_x} p_x(\Ba) \ = \ p_x(\Ba) \ .
\end{align}
\end{small}
In the rest of the paper, we will drop the explicit parameterization with $\cX$ and $\cY$
for all estimates to unclutter notation, and instead just use the hat sign to denote estimates.
In the same fashion as for the distributions, when we use
fixed samples $\cX$ and $\cY$, we obtain the following unbiased
estimates for the potential, energy and field given by Eq.~\eqref{eq:Plummer}, Eq.~\eqref{eq:field}, and Eq.~\eqref{eq:energy}:
\begin{small}
\begin{align}
\label{eq:objective1}
\hat{\Phi}(\Ba) \ &= \
\frac{1}{N_y} \sum_{i=1}^{N_y} k\left(\Ba,\By_i \right) \ - \
\frac{1}{N_x} \sum_{i=1}^{N_x} k\left(\Ba,\Bx_i \right) \ , \\
\label{eq:objectiveF}
\hat{F}\left(\rho\right) \ &= \ \frac{1}{2}
\ \left(
\frac{1}{N_y^2} \sum_{i=1}^{N_y} \sum_{j=1}^{N_y}
k \left( \By_i,\By_j \right)
\ - \
\frac{2}{N_y N_x} \sum_{i=1}^{N_y} \sum_{j=1}^{N_x}
k \left( \By_i,\Bx_j \right)
\ + \
\frac{1}{N_x^2} \sum_{i=1}^{N_x} \sum_{j=1}^{N_x}
k \left( \Bx_i,\Bx_j \right) \right) \\ \nonumber
&= \ \frac{1}{2}
\ \left( \frac{1}{N_y} \sum_{i=1}^{N_y} \hat{\Phi}\left(\By_i\right)
\ - \ \frac{1}{N_x} \sum_{i=1}^{N_x} \hat{\Phi}\left(\Bx_i\right)
  \right) \ , \\
\label{eq:objective4}
\hat{\EE}(\Ba)
\ &= \ - \ \nab_{a} \ \hat{\Phi}(\Ba) \
\ = \ - \ \frac{1}{N_y} \sum_{i=1}^{N_y} \nab_{a} k\left(\Ba,\By_i \right) \ + \
\ \frac{1}{N_x} \sum_{i=1}^{N_x} \nab_{a} k\left(\Ba,\Bx_i \right) \\
\nonumber
\hat{\EE}(\By_i)  \ &= \ - \ N_y \ \nab_{y_i} \hat{F}\left(\rho\right) \ , \quad
\hat{\EE}(\Bx_i)  \ = \ N_x \ \nab_{x_i} \hat{F}\left(\rho\right) \ .
\end{align}
\end{small}
These are again unbiased, e.g.:
\begin{align}
\EXP_{\cX,\cY} (\hat{\Phi}(\Ba)) \ &= \
   \frac{1}{N_y} \ \sum_{i=1}^{N_y}  E_{\By_i}\left(k(\Ba,\By_i)\right)
\ - \ \frac{1}{N_x} \ \sum_{i=1}^{N_x}  E_{\Bx_i}\left(k(\Ba,\Bx_i)\right)
\\ \nonumber
&= \ \frac{1}{N_y} \ \sum_{i=1}^{N_y} \int p_y(\By_i) \ k(\Ba,\By_i) \ \Rd \By_i
\ -\ \frac{1}{N_x} \ \sum_{i=1}^{N_x} \int p_x(\Bx_i) \ k(\Ba,\Bx_i) \ \Rd \Bx_i \\ \nonumber
&= \ \int p_y(\By) \ k(\Ba,\By) \ \Rd \By
\ - \ \int p_x(\Bx) \ k(\Ba,\Bx) \ \Rd \Bx \\ \nonumber
&= \
   \int \left(p_y(\Bx) \ - \ p_x(\Bx)\right) \ k(\Ba,\Bx) \ \Rd \Bx 
\ = \
   \int \rho(\Bx) \ k(\Ba,\Bx) \ \Rd \Bx \ = \ \Phi(\Ba)
\end{align}
If we draw samples of infinite size, all these expressions for a fixed sample
size lead to the equivalent statements for densities.
The sample-based formulation, that is, point charges in physical
terms, can only have local energy minima or maxima
at locations of samples \citep{Dembo:88}.
Furthermore the field lines originate and end at samples, therefore
the field guides model samples $\Bx$ toward real
world samples $\By$, as depicted in Figure~\ref{fig:field}.
The factors $N_y$ and $N_x$ in the last equations arise from the fact
that $-\nab_{a} F$ gives the force which is applied to a sample
with charge.
A sample $\By_i$ is positively charged with $1/N_y$ and follows
$- \nab_{y_i} F$ while a sample $\Bx_i$ is negatively charged
with $-1/N_x$ and therefore follows $- \nab_{x_i} F$, too.
Thus, following the force induced on
a sample by the field is equivalent to gradient descent of the energy
$F$ with respect to samples $\By_i$ and $\Bx_i$.

\subsection{Mixture of Gaussians}
\label{sec:gaussianmixture}
We use the synthetic data set introduced by
\citet{Lim:17} to show that
Coulomb GANs avoid mode collapse
and that all modes of the target distribution are captured by the
generative model.
This data set comprises 100K data points drawn from a Gaussian mixture
model of 25 components which are spread out evenly in the range
$[-21, 21]\times [-21, 21]$, with each component having a variance of 1.
To make results comparable with \citet{Lim:17}, the Coulomb GAN used a
discriminator network with 2 hidden layers of 128 units,
however we avoided batch normalization by using the ELU activation
function \citep{Clevert:16elu}.
We used the Plummer kernel in 3 dimensions ($d=3$)
with an epsilon of 3 ($\epsilon=3$) and
a learning rate of 0.01, both of which were exponentially decayed during the 1M
update steps of the Adam optimizer.

As can be seen in Figure~\ref{fig:gaussianmixture}, samples from the learned Coulomb GAN
very well approximate the target distribution.
All components of the original distribution are present at the model
distribution at approximately the
correct ratio, as shown in Figure~\ref{fig:gm2d-hist}.
Moreover, the generated samples
are distributed approximately according to
the same spread for each component of the real world distribution.
Coulomb GANs outperform
other compared methods, which either fail to learn the distribution completely,
ignore some of the modes, or do not capture the within-mode spread
of a Gaussian. The Coulomb GAN is the
only GAN approach that manages to avoid a within-cluster
collapse leading to insufficient variance within a cluster.

\begin{figure}[hb]
\centering
\begin{subfigure}[t]{0.195\textwidth}
    \centering
    \includegraphics[width=\textwidth]{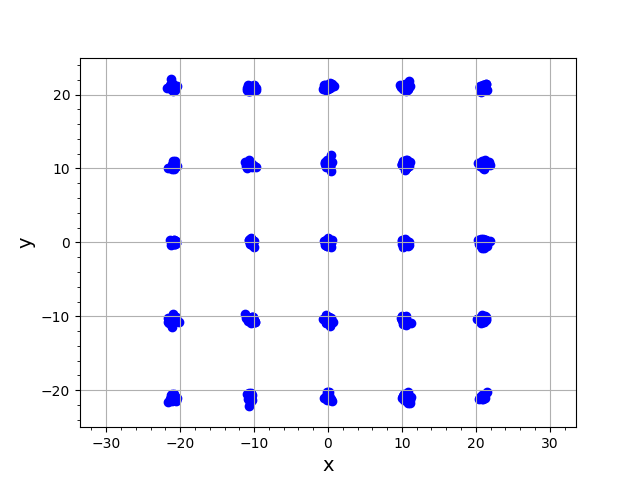}
    \caption{True data}
\end{subfigure}
\begin{subfigure}[t]{0.195\textwidth}
    \centering
    \includegraphics[width=\textwidth]{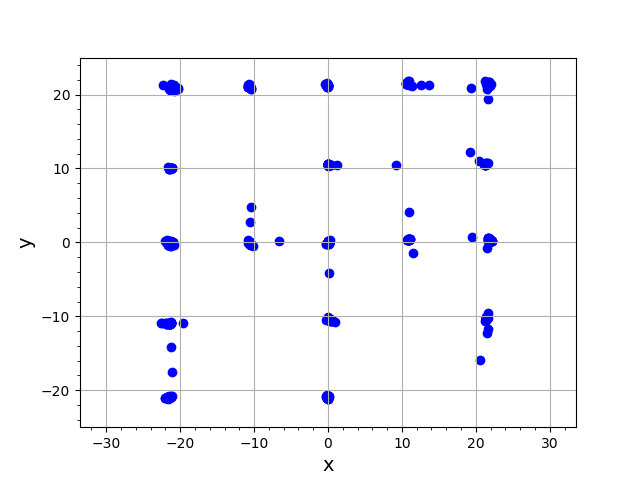}
    \caption{GAN}
\end{subfigure}
\begin{subfigure}[t]{0.195\textwidth}
    \centering
    \includegraphics[width=\textwidth]{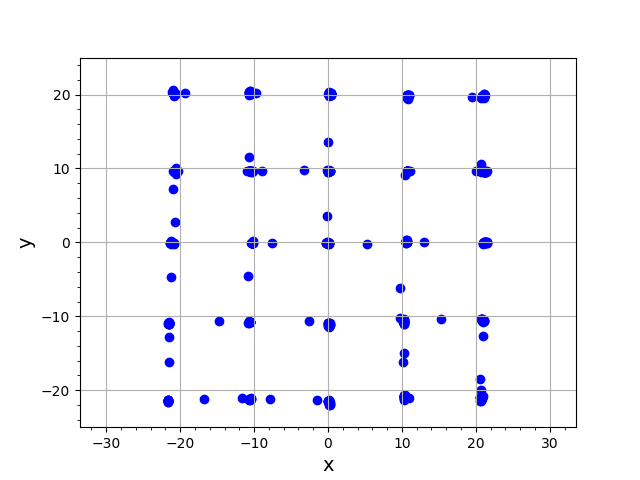}
    \caption{Geometric GAN}
\end{subfigure}
\begin{subfigure}[t]{0.195\textwidth}
    \centering
    \includegraphics[width=\textwidth]{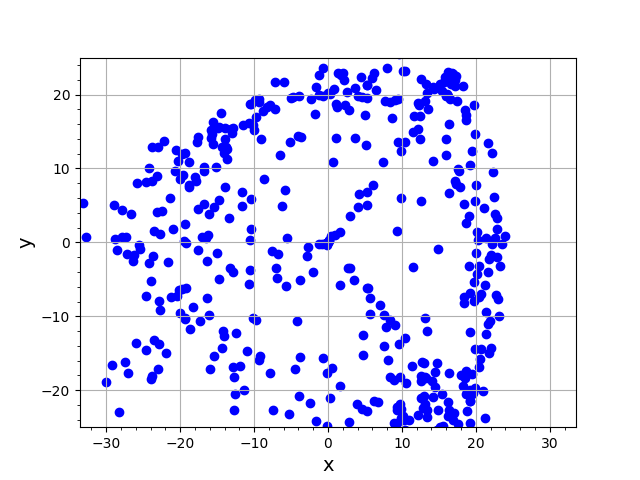}
    \caption{WGAN}
\end{subfigure}
\begin{subfigure}[t]{0.195\textwidth}
    \centering
    \includegraphics[width=\textwidth]{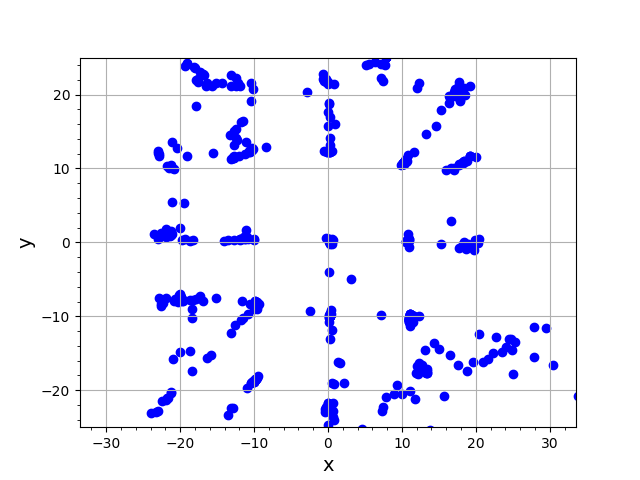}
    \caption{meanGAN + proj.}
\end{subfigure}
\begin{subfigure}[t]{0.195\textwidth}
    \centering
    \includegraphics[width=\textwidth]{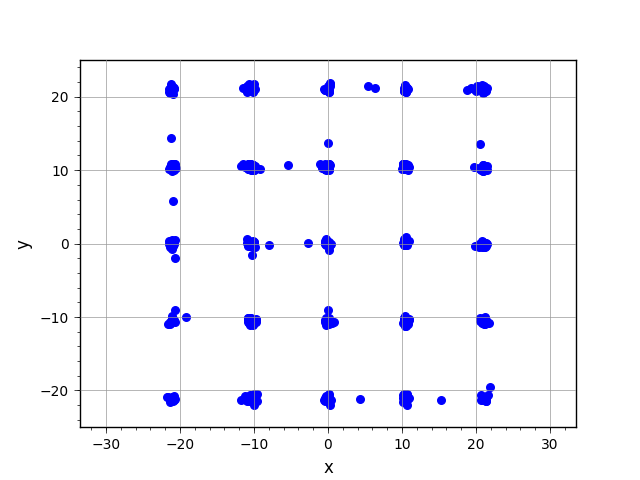}
    \caption{Coulomb GAN}
\end{subfigure}
\begin{subfigure}[t]{0.195\textwidth}
    \centering
    \includegraphics[width=\textwidth]{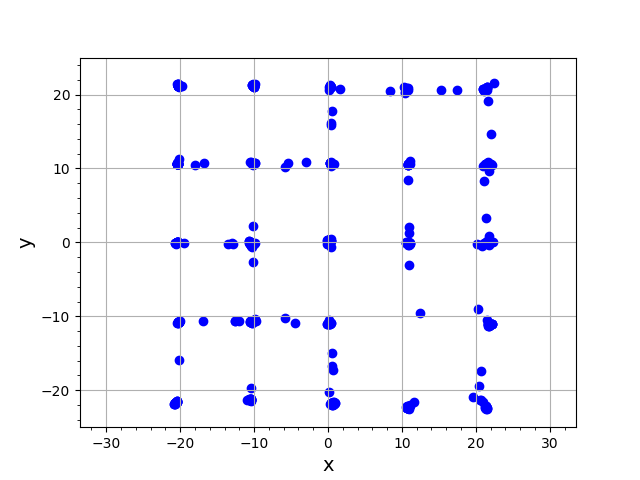}
    \caption{GAN + WD}
\end{subfigure}
\begin{subfigure}[t]{0.195\textwidth}
    \centering
    \includegraphics[width=\textwidth]{./fake_samples_toy4_l2wgan_wproj_toy4_rmsprop_1.png}
    \caption{Geo. GAN + WD}
\end{subfigure}
\begin{subfigure}[t]{0.195\textwidth}
    \centering
    \includegraphics[width=\textwidth]{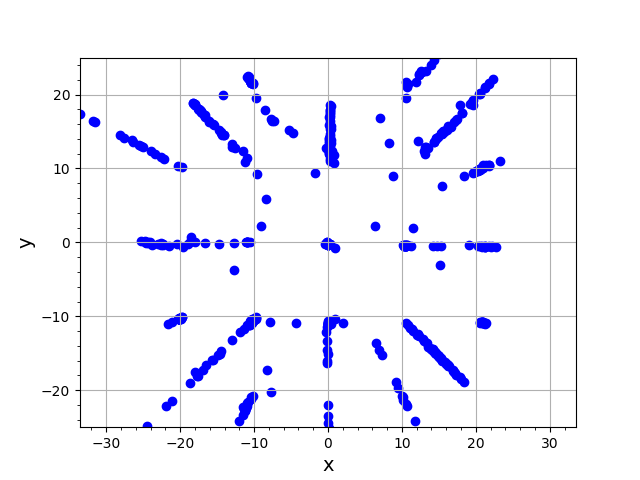}
    \caption{WGAN + WD}
\end{subfigure}
\begin{subfigure}[t]{0.195\textwidth}
    \centering
    \includegraphics[width=\textwidth]{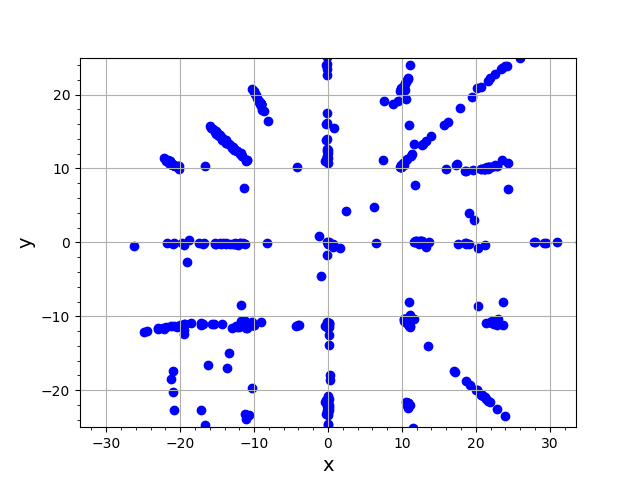}
    \caption{meanGAN + WD}
\end{subfigure}
\caption{Scatter plots of generated samples from
  different GAN variants for the mixture of 25 Gaussians and the true
  data distribution. ``WD'' indicates weight decay and ``proj.'' means
  projection. Results and images for all methods except the Coulomb GAN are taken
  from \citet{Lim:17}.}
\label{fig:gaussianmixture}
\end{figure}
\begin{figure}[hb]
\centering
    \includegraphics[width=0.8\textwidth]{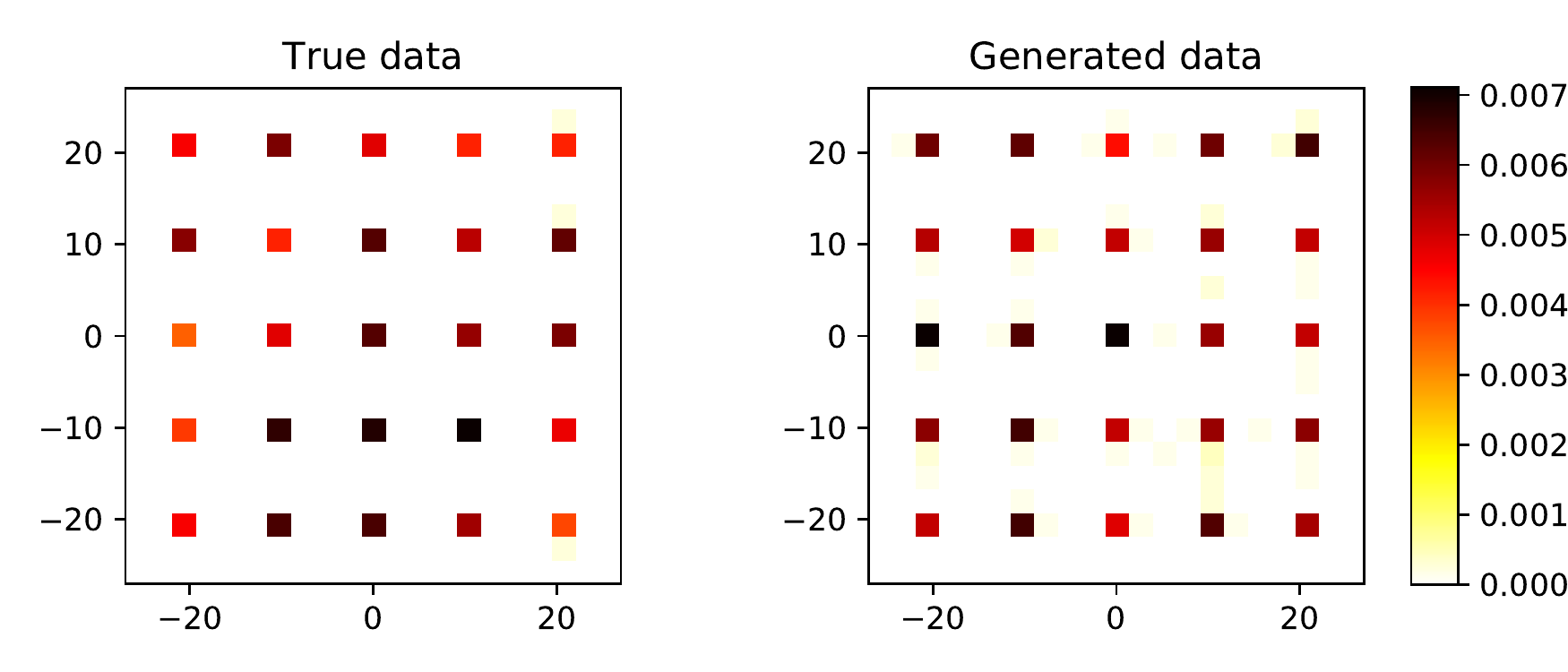}
\caption{
  2D histogram of the density of generated and the training st
  data for the mixture of 25 Gaussians.
  For constructing the histogram,
  10k samples were drawn from the target and the model distribution.
  The Coulomb GAN captures the underlying distribution well,
  does not miss any modes, and places almost all probability mass on
  the modes. Only the Coulomb GAN captured the within-mode spread
  of the Gaussians.}
\label{fig:gm2d-hist}
\end{figure}

\clearpage

\subsection{Pseudocode for Coulomb GANs}
The following gives the pseudo code for training GANs.
Note that when calculating the derivative of $\hat{\Phi}(\Ba_i; \mathcal{X, Y})$, it is
important to only derive with respect to $\Ba$, and not wrt. $\mathcal{X, Y}$,
even if it can happen that e.g. $\Ba \in \mathcal{X}$. In frameworks that offer
automatic differentiation such as Tensorflow or Theano, this means stopping
the possible gradient back-propagation through those parameters.

\begin{algorithm}[h]
\caption{\small Minibatch stochastic gradient descent training of Coulomb GANs
for updating the the discriminator weights $\Bw$ and the generator weights $\Bth$.}
\begin{algorithmic}
\label{alg:coulombgan}
\WHILE{Stopping criterion not met}{
    \STATE{$\bullet$ Sample minibatch of $N_x$ training samples $\{ \Bx_1, \dots, \Bx_{N_x} \}$ from training set}
    \STATE{$\bullet$ Sample minibatch of $N_y$ generator samples $\{ \By_1, \dots, \By_{N_y} \}$ from the generator}
    \STATE{$\bullet$ Calculate the gradient for the discriminator weights: \begin{small}
        \[
            d\Bw \gets \nabla_{\Bw} \left[
            \frac{1}{2} \sum_{i=1}^{N_x} \left(D(\Bx_i) - \hat{\Phi}(\Bx_i)\right)^2 +
            \frac{1}{2} \sum_{i=1}^{N_y} \left(D(\By_i) - \hat{\Phi}(\By_i)\right)^2
            \right]
        \]\end{small} }
    \STATE{$\bullet$ Calculate the gradient for the generator weights:
        \[ \begin{small}
            d\Bth \gets \nabla_{\Bth}\left[ -\frac{1}{2} \frac{1}{N_x} \sum_{i=1}^{N_x} D\left(\Bx_i\right)\right]
           \end{small}
        \]}
        }
    \STATE{$\bullet$ Update weights according to optimizer rule (e.g. Adam):
        \begin{small}
        \begin{align*}
        \Bw_{n+1} &= \Bw_n + \texttt{ADAM}(d\Bw, n) \\
        \Bth_{n+1} &= \Bth_n +\texttt{ADAM}(d\Bth, n)
        \end{align*}
        \end{small}}

\ENDWHILE
\end{algorithmic}
\end{algorithm}

\clearpage

\subsection{More Samples from Coulomb GANs}
\label{sec:more-examples}
\subsubsection*{CelebA}
Images from a Coulomb GAN after training on CelebA data set.
The low FID stems from the fact that the images show a wide variety
of different faces, backgrounds, eye colors and orientations.
\begin{center}
\includegraphics[width=0.75\textwidth]{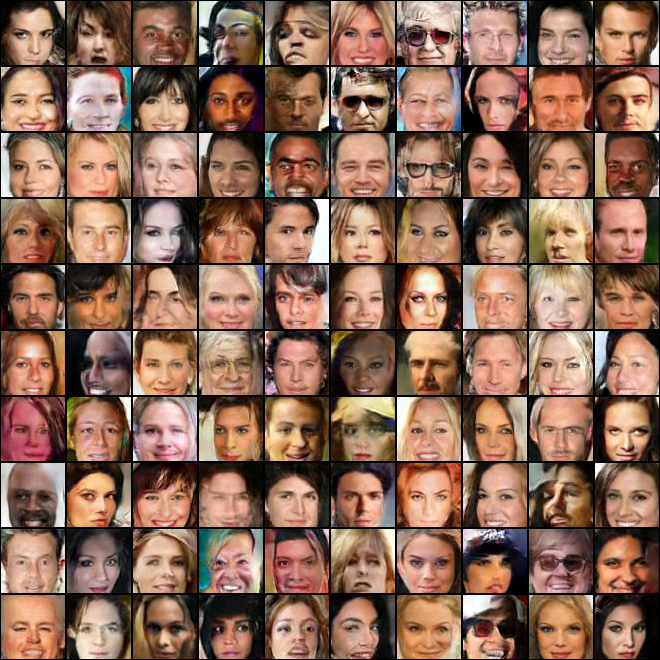}
\end{center}

\clearpage

\subsubsection*{LSUN bedrooms}
Images from a Coulomb GAN after training on the LSUN
      bedroom data set.
\begin{center}
\includegraphics[width=0.75\textwidth]{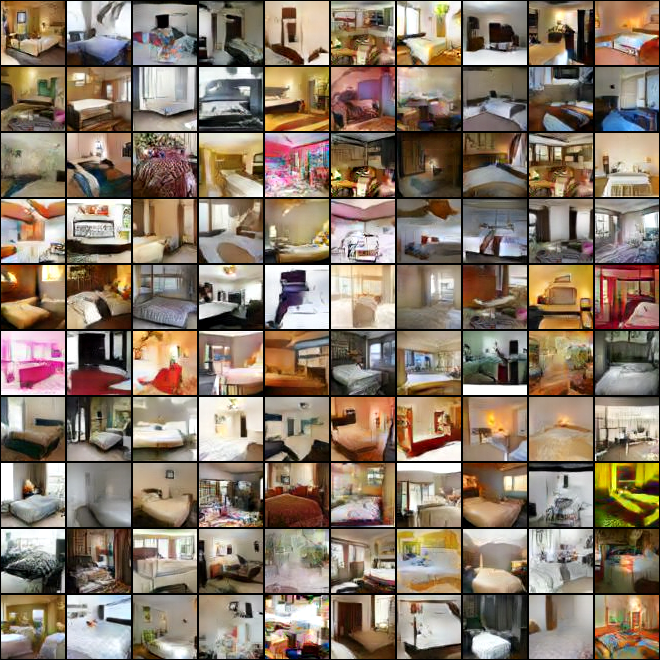}
\end{center}

\subsubsection*{CIFAR 10}
Images from a Coulomb GAN after training on the CIFAR 10 data set.
\begin{center}
    \includegraphics[width=0.75\textwidth]{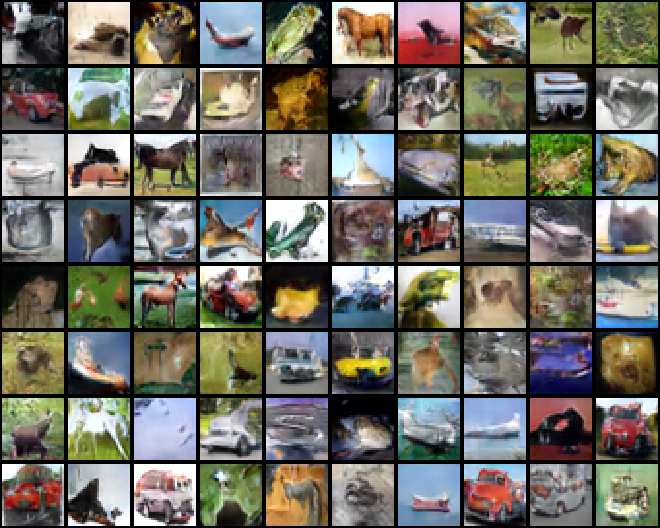}
\end{center}

\end{document}